\documentclass[letterpaper, 10 pt, conference]{ieeeconf}  % Comment this line out if you need a4paper

\IEEEoverridecommandlockouts                              % This command is only needed if
\usepackage{amsmath} % assumes amsmath package installed
\usepackage{amssymb}  % assumes amsmath package installed
\usepackage{graphicx}
\usepackage{float}
\usepackage{cite}
\usepackage{url}
\usepackage{booktabs}
\usepackage{xcolor}

\usepackage{algorithm}
\usepackage[noend]{algpseudocode}
\usepackage[hidelinks]{hyperref}
\usepackage{subfig}
\usepackage[font={small,it}]{caption}
\usepackage[noabbrev]{cleveref}
\usepackage{cuted}
\usepackage{amsthm}

\makeatletter
\let\OldStatex\Statex
\renewcommand{\Statex}[1][3]{%
  \setlength\@tempdima{\algorithmicindent}%
    \OldStatex\hskip\dimexpr#1\@tempdima\relax}
\makeatother

\urldef{\smith}\url{stephen.smith@uwaterloo.ca}
\urldef{\deiaco}\url{ryan.deiaco@uwaterloo.ca}
\urldef{\czarnecki}\url{kczarnec@gsd.uwaterloo.ca}

\newtheorem{theorem}{Theorem}

\theoremstyle{definition}
\newtheorem{definition}{Definition}[section]

\title{\LARGE \bf
Universally Safe Swerve Manoeuvres for Autonomous Driving
}

\author{Ryan De Iaco, Stephen L. Smith, and Krzysztof Czarnecki% <-this % stops a space
\thanks{This work was supported by the Government of Ontario.}% <-this % stops a space
\thanks{The authors are with the Department of Electrical and Computer Engineering, University of Waterloo, Waterloo ON, N2L 3G1, Canada (\deiaco; \smith, \czarnecki)}}

\begin{document}

\maketitle
\thispagestyle{empty}
\pagestyle{empty}

%%%%%%%%%%%%%%%%%%%%%%%%%%%%%%%%%%%%%%%%%%%%%%%%%%%%%%%%%%%%%%%%%%%%%%%%%%%%%%%%
\begin{abstract}

This paper characterizes safe following distances for on-road driving 
when vehicles can avoid collisions by either braking or by swerving 
into an adjacent lane.
In particular, we focus on safety as defined in 
the Responsibility-Sensitive Safety (RSS) framework.
We extend RSS by introducing swerve manoeuvres as a valid
response in addition to the already present brake manoeuvre. These
swerve manoeuvres use the more realistic kinematic bicycle model rather than 
the double integrator model of RSS. When vehicles are able to swerve and brake, it is shown 
that their required safe following distance at higher speeds is less than
that required through braking alone. In addition, when all vehicles follow this new distance, they are provably safe. The use of the kinematic bicycle model is then validated by comparing these swerve manoeuvres to that of a dynamic single-track model.

\end{abstract}

%%%%%%%%%%%%%%%%%%%%%%%%%%%%%%%%%%%%%%%%%%%%%%%%%%%%%%%%%%%%%%%%%%%%%%%%%%%%%%%%
\section{Introduction}

The main bottleneck for the public acceptance and ubiquity of autonomous driving is
the current lack of safety guarantees. There are three
main ways to establish the safety of an autonomous vehicle. The first involves
measuring crash statistics over a large number of autonomously driven kilometres and
comparing them to the equivalent human rates for each category of collision severity. 
However, particularly with severe collisions, the number of kilometres
required to establish a statistically significant collision rate renders this method
impractical for establishing safety.

An alternative method for determining the safety of a system is through scenario-based
verification~\cite{thorn_kimmel_chaka_2018}. 
This method uses a set of scenarios that
validate the vehicle's behaviour across a representative set of situations. The goal
is for the set of scenarios to capture most of the required driving behaviour
necessary for safe driving. However, it is difficult to construct such a set of 
scenarios that captures all of the challenging conditions faced by an autonomous vehicle~\cite{abeysirigoonawardena_shkurti_dudek_2019}.

A third approach for verifying the safety of a system is formally proving the behaviour
of a vehicle is safe~\cite{althoff_althoff_wollherr_buss_2010,althoff_dolan_2014,
DBLP:journals/corr/abs-1708-06374,leung_schmerling_chen_talbot_gerdes_pavone_2018}.
In order to compute useful safety bounds, these works often include simplifying assumptions.
The difficulty with this 
method lies in selecting reasonable assumptions to make. Generally, the stronger the assumptions
made, the easier to prove the system is safe. However, if the assumptions are too strong, 
they may not hold in general driving scenarios. An additional challenge with this method
is that to prove safety, the driving behaviour may need to be conservative, or highly
restrictive.

\begin{figure}[thpb]
  \centering
  \subfloat[\label{fig:brake_example}]{\includegraphics[scale=0.35]{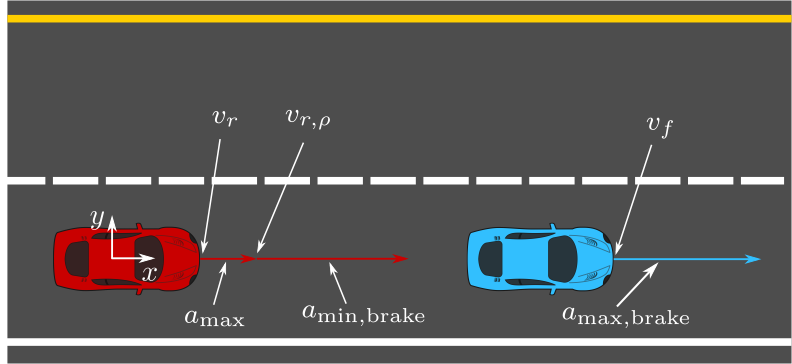}}\qquad
  \subfloat[\label{fig:swerve_example}]{\includegraphics[scale=0.35]{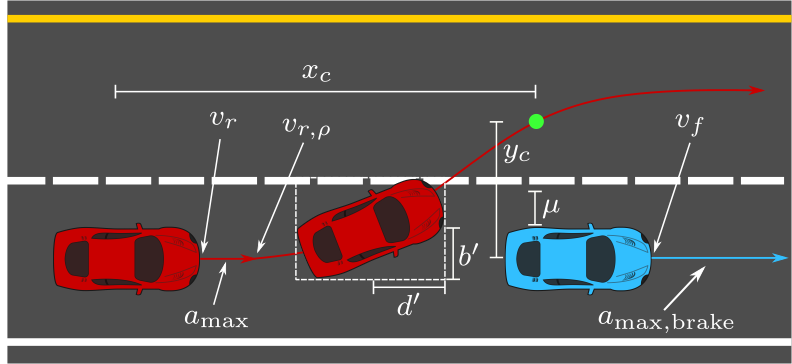}}\qquad
  \subfloat[\label{fig:brake_for_swerve_example}]{\includegraphics[scale=0.35]{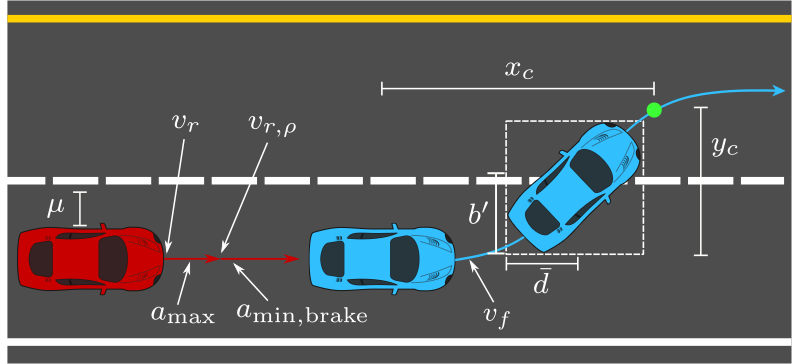}}\qquad
  \subfloat[\label{fig:swerve_for_swerve_example}]{\includegraphics[scale=0.35]{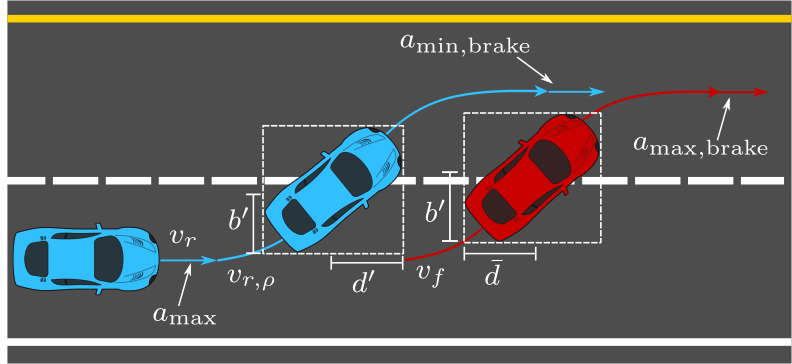}}\qquad
  \caption{
    \label{fig:swerve_brake_comparison}
    (a) The standard RSS braking manoeuvre for a braking leading vehicle. Velocity and
    acceleration arrows point to path segments where they occur.
    (b) The proposed swerve manoeuvre for a leading braking vehicle. The green dot represents the
    lateral clearance distance $y_c$ required by RSS.
    (c) The braking manoeuvre required for a swerving leading vehicle.
    (d) The swerving manoeuvre required for a swerving leading vehicle.
  }
\end{figure}

This paper aims to address the latter issue, especially as it pertains to the 
Responsibility-Sensitive Safety (RSS) framework~\cite{DBLP:journals/corr/abs-1708-06374}.
Fundamental to the RSS framework is its assumption of
responsibility, and that vehicles have a duty of care to one another.
The assumption of responsible
behaviour allows for the autonomous vehicle to make meaningful progress in the driving task. 
Under other frameworks that assume adversarial vehicles, the autonomous vehicle often exhibits 
over-conservative behaviour that impedes progress.
This assumption of responsible behaviour allows for the computation of safe following distances
such that vehicles can 
comfortably brake for a braking vehicle in front of them, without causing a collision. 
This following distance is a function of both vehicle's speeds and maximum accelerations, as well as 
the reacting vehicle's reaction time. When computing this following distance, the vehicles are modeled by 
a kinematic particle model.
As long as this following distance is maintained by all vehicles, no collisions can occur.

This paper extends this framework to include swerve manoeuvres feasible
for the kinematic bicycle model as a valid response, in addition to the standard braking response. In doing so, vehicles are able to follow at
closer following distances at higher speeds, allowing for more efficient use of the road network. Using swerve manoeuvres feasible for the kinematic bicycle model ensures that these manoeuvres are more realistic than those possible under the particle model used in RSS.

\subsection{Contributions}

The contributions in this paper are as follows. The first is the derivation of safe following distances in scenarios where vehicles perform swerve manoeuvres.
In addition to the scenario of braking for a braking vehicle considered in 
RSS, we consider the additional scenarios of swerving for a braking vehicle, 
braking for a swerving vehicle, and swerving 
for a swerving vehicle. These distances are then used to form a novel 
extension of the RSS framework: a shorter following distance at high
speeds. When all vehicles in a road network follow this new distance 
as well the rest of the RSS framework, 
and satisfy our assumptions, they are provably safe from collision.

The second contribution is a validation of our use of the kinematic bicycle model by comparing our
swerve manoeuvres to manoeuvres generated under a dynamic single-track model~\cite{gerdts_2005}.
As part of this dynamic model, we include a Pacejka tire model~\cite{hans_bakker_1992} to account for road surface traction.
We show that the kinematic model, when lateral acceleration is constrained, can accurately
estimate the longitudinal distance required to perform swerve manoeuvres
using the dynamic model. 

\subsection{Related Work}\label{sec:related_work}

Previous work on swerve manoeuvres for autonomous driving have often focused on feasible
manoeuvres according to various kinodynamic models~\cite{schmidt_oechsle_branz_2006}.
In particular, many of these papers have assumed some variant of the bicycle 
model~\cite{shiller_zvi_sundar_satish_2005,shiller_sundar_1996,shiller_sundar_1995,
dingle_guzzella_2010} and performed optimization to generate optimal swerve manoeuvres.
However, under these models the optimal solution is not generated through a closed form
solution, which makes formally proving safety challenging.

Other work has instead simplified the vehicle model to a point mass 
model~\cite{shiller_sundar_1998,jula_kosmatopoulos_ioannou_2000,pek_zahn_althoff_2017} in order 
to yield closed form, optimal solutions. However, this comes at the cost of the nonholonomic
constraint present in the bicycle model, which result in manoeuvres that would be 
unrealistic for a car to execute.
The goal with this work is to yield closed form, feasible solutions to swerve
manoeuvre boundary condition problems, while still preserving the kinematic constraints that 
allow the manoeuvre to be executable by a real vehicle.

Previous work on using the kinematic bicycle model for autonomous driving has shown it is an 
effective model for tracking trajectories in MPC~\cite{kong_pfeiffer_schildbach_borrelli_2015}, and as such, contains important kinematic constraints that capture some of the limits of vehicle motion. 
Past work has also shown that the kinematic bicycle model is an accurate approximation to vehicle motion at low accelerations~\cite{polack_altche_dandrea-novel_fortelle_2017}, which we expect to see as well in our validation.

\section{Preliminaries}

\subsection{Responsibility-Sensitive Safety (RSS)}
\label{sec:preliminaries}

In this paper, we rely on two aspects of the RSS framework; the longitudinal and lateral
safe distances required between two vehicles. In particular, we examine how the
equivalent longitudinal safe distance for a swerve manoeuvre compares to that of a
brake manoeuvre, while maintaining an appropriate lateral safe distance when required.
In this work, we compare swerve manoeuvres moving to the left as in Figure~\ref{fig:swerve_brake_comparison}, however, the same analysis can be applied to swerves moving to the right.

In RSS, safe distances are a function of several variables that describe the situation. The initial speed of the
rear autonomous vehicle is given by $v_r$, and the initial speed of the front vehicle is denoted
by $v_f$. The reaction time is given by $\rho$. The interpretation of the reaction time is the duration
after which a vehicle can apply a mitigating action. During the reaction time, both
vehicles apply the most dangerous acceleration possible, $a_{\max,\text{accel}}$, $a_{\max,\text{brake}}$ in 
the longitudinal case, and $a_{\max}^{\text{lat}}$ in the lateral case. To ensure passenger comfort, as well as to
prevent tailgater safety issues, the mitigating
reaction of the rear vehicle is assumed to be a comfortable deceleration, denoted $a_{\min,\text{brake}}$. This term comes from RSS, and is interpreted as the threshold for a safe, responsible braking response for the autonomous car. Note that these accelerations are magnitudes. 

We denote the positive part of an expression with $[\cdot]_+$.
Velocities are signed according to Figure~\ref{fig:brake_example},
and accelerations are unsigned parameters of the framework.
If the post-reaction speeds $v_{r, \rho}$ and $v_{f, \rho}$ are given by
\begin{equation} v_{r,\rho} = v_r + a_{\max,\text{accel}} \rho, \end{equation}
\begin{equation} v_{r,\rho}^{\text{lat}} = v_{r}^{\text{lat}} - a_{\max}^{\text{lat}} \rho, \end{equation}
\begin{equation} v_{f,\rho}^{\text{lat}} = v_{f}^{\text{lat}} + a_{\max}^{\text{lat}} \rho, \end{equation}
the \textit{longitudinal and lateral safe distances} are given by
\begin{multline}
d_{\text{long}} = \biggl[ v_r \rho + \frac{1}{2} a_{\max,\text{accel}}\rho^2 + \\ 
\frac{(v_r + v_{r,\rho})^2}{2 a_{\min,\text{brake}}} - \frac{v_f^2}{2 a_{\max,\text{brake}}} \biggr]_+ \label{eq:rss_long_dist},
\end{multline}
\begin{multline}
d_{\text{lat}} = \mu + \biggl[ -\left(\frac{v_{r}^{\text{lat}} + v_{r,\rho}^{\text{lat}}}{2}\right)\rho +
\frac{(v_{r,\rho}^{\text{lat}})^2}{2a_{\min}^{\text{lat}}} + \\ \frac{v_{f}^{\text{lat}} + v_{f,\rho}^{\text{lat}}}{2}\rho + \frac{(v_{f,\rho}^{\text{lat}})^2}{2a_{\min}^{\text{lat}}} \biggr]_+. \label{eq:rss_lat_dist}
\end{multline}

The longitudinal safe distance is between the frontmost point of the rear vehicle and the rearmost point of the front vehicle along the longitudinal direction,
and the lateral safe distance is between the rightmost point of
the rear vehicle and the leftmost point of the front vehicle along the lateral direction. These are left implicit in the original
RSS formulation, but since swerves involve rotation of the chassis, we make them explicit in this work.
The longitudinal safe distance $d_{\text{long}}$ is the distance required such that the rear vehicle can maximally accelerate during its
reaction time, then minimally decelerate to a stop, all while the front vehicle is maximally braking, without
causing a collision.
The lateral safe distance $d_{\text{lat}}$ is the distance required such that both vehicles can maximally accelerate towards each
other during the reaction time $\rho$, then minimally decelerate until zero lateral velocity, while
still maintaining at least a $\mu$ distance buffer.

When computing safety for swerve manoeuvres, the vehicle must maintain these safe distances with other relevant vehicles. These vehicles are relevant according to longitudinal and lateral adjacency, as defined below.
We denote the vehicle dimensions $d_f$, $d_r$, $b_l$, $b_r$ as in Figure~\ref{fig:outer_approx}.
\begin{definition}
If $x_1$, $x_2$ denote the longitudinal position of each vehicle, and then the vehicles are \textit{laterally adjacent} if $x_2 - d_r - d_f \leq x_1 \leq x_2 + d_r + d_f$.
\end{definition}
\begin{definition}
If $y_1$, $y_2$ denotes the lateral position of each vehicle, then the vehicles are  
\textit{longitudinally adjacent} if $y_2 - b_l - b_r - d_{lat} \leq y_1 \leq y_2 + b_l + b_r + d_{lat}$.
\end{definition}
Combining the definitions for safe distances and adjacency gives us a definition of safety.
\begin{definition}
A vehicle is \textit{laterally/longitudinally safe} from another vehicle if it is not laterally/longitudinally adjacent 
to the other vehicle, or if it
is laterally/longitudinally adjacent to the other vehicle and there is at least $d_{\text{lat}}$/$d_{\text{long}}$ 
of distance between them. 
\end{definition}

\begin{definition}
For a swerving vehicle and a non-swerving vehicle, as well as a given swerve manoeuvre, we define the \textit{lateral clearance distance}, $y_c$, as the earliest point in the swerve at which the swerving vehicle is no longer longitudinally adjacent to the non-swerving vehicle. 
\end{definition}
In Figure~\ref{fig:swerve_brake_comparison}, $y_c$ is reached at the green dot along the swerve.
The lateral clearance distance allows us to compute the longitudinal distance covered
by the swerve, which is denoted by $x_c$. We then use $x_c$ to compute the equivalent of
$d_{\text{long}}$ for a swerve manoeuvre, and compare it to 
\Cref{eq:rss_long_dist}.

\subsection{Vehicle Models}
The analysis in this paper relies upon three different kinodynamic models. The first is the 
particle kinematic model, which is used in the RSS framework. Through all of
these kinematic models, $x$ is longitudinal displacement and $y$ is lateral displacement. 
The control input is the acceleration in each dimension
\begin{align}
\label{eq:pm}
\ddot{x} = a_x, && \ddot{y} = a_y.
\end{align}

\begin{figure}
  \centering
  \subfloat[\label{fig:bicycle_model}]{\includegraphics[scale=0.4]{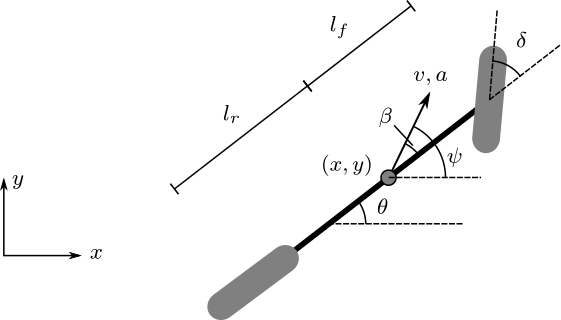}}\qquad
  \subfloat[\label{fig:dynamic_model}]{\includegraphics[scale=0.4]{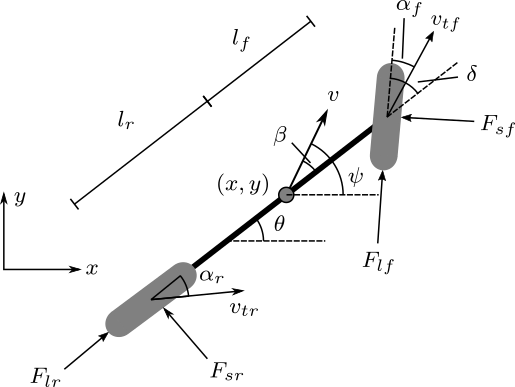}}\qquad
  \caption{
    \label{fig:kinematic_models}
    (a) The kinematic bicycle model, along with its associated variables.
    (b) The dynamic single track model used for validation~\cite{gerdts_2005}. Drag forces are
    omitted for simplicity, but are included in our computation.
  }
\end{figure}

When computing swerve manoeuvres, we wish to model the non-holonomic constraints on a car's motion to make the manoeuvres realistic. To do so, we rely on the kinematic bicycle
model, a model commonly used in autonomous 
driving~\cite{snider_2009,luca_oriolo_samson_1998,kong_pfeiffer_schildbach_borrelli_2015}.
This is illustrated in Figure~\ref{fig:bicycle_model}.
In this model, $v$ is the velocity of the vehicle, $\psi$ is the heading of velocity at the centre of mass,
$\theta$ is the yaw of the chassis,
$\beta$ is the slip angle of the centre of mass relative to the chassis, $a$ is the input 
acceleration, $\delta$ is the input steering angle, $R_c$ is the turning radius of the center of mass, 
and $l_r$ and $l_f$ are the distances from 
the rear and front axle to the centre of mass, respectively

\begin{align}
\label{eq:bicycle_model}
\dot{x} &= v\cos(\psi + \beta), & \beta &= \tan^{-1}\left( \frac{l_r}{l_r + l_f} \tan(\delta) \right), \nonumber \\
\dot{y} &= v\sin(\psi + \beta),& \theta &= \psi - \beta, \nonumber \\
\dot{\theta} &= \frac{v\tan(\delta)}{l_r + l_f} , & |\delta| &\leq \delta_{\max} \nonumber \\
\dot{v} &= a, & |a_{\text{lat}}| &= \frac{v^2}{R_c} \leq a^{\text{lat}}_{\min}, \nonumber \\
R_c &= \frac{l_r + l_f}{\cos(\beta)\tan(\delta)}, & -a_{\text{brake},\min} &\leq a \leq a_{\max}.   
\end{align}

Finally, to verify our kinematic approximation is valid, we compare our swerve manoeuvres
to those executed by a dynamic single-track vehicle model~\cite{gerdts_2005} with tires 
modelled using the Pacejka tire model~\cite{hans_bakker_1992}. This model is shown in
Figure~\ref{fig:dynamic_model}.
In this vehicle model, $v$, $\psi$, $\beta$, $\delta$, $l_f$, and $l_r$ are the same as 
the bicycle model. The slip angles of the front and rear tires are $\alpha_f$ and $\alpha_r$, respectively. The lateral tire forces on the front and rear tires are denoted $F_{\textit{sf}}$ and $F_{\textit{sr}}$, respectively, and
$F_{\textit{lf}}$ and $F_{\textit{lr}}$ denote the longitudinal tire forces at the front and rear tires, 
respectively. The drag mount point is denoted $e_{\textit{SP}}$, and $F_{\textit{Ax}}$ and $F_{\textit{Ay}}$ are the longitudinal
and lateral drag forces, respectively. The yaw rate is $\omega_z$, and $\omega_{\delta}$ is the input 
steering rate. The mass of the car is $m$, and $I_{\textit{zz}}$ is the inertia about the $z$-axis.
We omit the equations of motion for brevity, but they are presented in the 
reference~\cite{gerdts_2005}.

\section{Problem Formulation}\label{sec:problem_formulation}

The fundamental problem this paper addresses is to compute the longitudinal safe distance required when there is a free lane (or shoulder)
to the left or right of the vehicle, allowing for an evasive swerve manoeuvre. This requires knowing the longitudinal safe distance required for the scenarios illustrated in Figure~\ref{fig:swerve_brake_comparison}.
As can be seen, when computing the longitudinal safe distances for swerves, one
needs to consider both longitudinal and lateral clearance, since swerves involve lateral
and longitudinal displacement.

Since vehicles rotate during swerves, rotation must be compensated for when
computing these clearances. After compensating for rotation, the longitudinal swerve distance $x_c$ can 
then be used to compute the longitudinal safe distance required for a swerve. 
In RSS, safety was proved for a particle model. This paper
extends those results to prove the safety for swerves feasible for the kinematic bicycle model.
It is then shown how this result can be applied to more general models in 
Section~\ref{sec:validation_and_results}.
This task then breaks down into five subproblems.

\textbf{Subproblem 1.} Given the initial speed of a swerving vehicle $v_r$, 
the vehicle dimensions $d_f$, $d_r$, $b_l$, $b_r$ as in Figure~\ref{fig:outer_approx}, 
and parameters $\mu$ and $\rho$, compute a lateral clearance distance $y_c$ sufficient for lateral safety when a swerving vehicle becomes laterally adjacent to a lead vehicle. 

\textbf{Subproblem 2.} Given the kinematic constraints in (\ref{eq:bicycle_model}),
the initial vehicle speeds $v_r$ and $v_f$, the lateral clearance distance $y_c$, and parameters $\rho$, 
$a_{\max}$, $a_{\min,\text{brake}}$, $a_{\max,\text{brake}}$, $a^{\text{lat}}_{\max}$, and
$a^{\text{lat}}_{\min}$, compute a longitudinal safe distance
sufficient for safety when swerving for a braking lead vehicle. This is illustrated in Figure~\ref{fig:swerve_example}.

\textbf{Subproblem 3.} Given the initial vehicle speeds $v_r$ and $v_f$, the clearance point 
$y_c$, and parameters $\rho$, $a_{\max}$, $a_{\min,\text{brake}}$, $a^{\text{lat}}_{\max}$, and 
$a^{\text{lat}}_{\min}$, compute 
a longitudinal safe distance sufficient for safety when braking for a swerving lead vehicle. This is illustrated in
Figure~\ref{fig:brake_for_swerve_example}.

\textbf{Subproblem 4.} Given the kinematic constraints in (\ref{eq:bicycle_model}),
the initial vehicle speeds $v_r$ and $v_f$, the parameters $\rho$, 
$a_{\max}$, $a_{\min,\text{brake}}$, $a_{\max,\text{brake}}$, $a^{\text{lat}}_{\max}$, and
$a^{\text{lat}}_{\min}$, compute a longitudinal safe distance
sufficient for safety when swerving behind a swerving lead vehicle. This is illustrated in
Figure~\ref{fig:swerve_for_swerve_example}.

\textbf{Subproblem 5.} Given longitudinal safe distance sufficient for safety
when swerving for a braking vehicle, braking for a swerving lead vehicle, and swerving
for a swerving lead vehicle, compute a longitudinal safe distance akin to $d_{\text{long}}$
that is sufficient for universal safety when maintained by all vehicles on the road.

The first subproblem is addressed in Section~\ref{sec:method_lat_clear}, the second in
Section~\ref{sec:method_swerve_for_brake}, the third in 
Section~\ref{sec:method_braking_for_swerving_vehicle},
the fourth in Section~\ref{sec:method_swerve_for_swerve},
and the fifth in Section~\ref{sec:method_universal_following_distance}.

The work in this paper makes the following assumptions on responsible behaviour:
\begin{enumerate}
\item \label{swerve_brake_assumption} A vehicle will only perform a swerve manoeuvre
if it is not braking, and will only perform a brake manoeuvre if it is
not swerving.

\item \label{reasonable_manoeuvre_assumption} For every swerve manoeuvre, each 
vehicle reaches the lateral clearance distance only once. As a result, once a vehicle has
committed to a lane change by reaching the lateral clearance distance, it will
not return to its previous lane.

\item \label{forward_motion_assumption} Each vehicle moves forward along the road,
$v \geq 0$ and $\frac{-\pi}{2} \leq \psi \leq \frac{\pi}{2}$.
\end{enumerate}

\section{Computing the Longitudinal Safe Distance}
\label{sec:method}

\subsection{Lateral Clearance Distance}
\label{sec:method_lat_clear}

To compute the lateral clearance distance $y_c$, we modify
Equation~\eqref{eq:rss_lat_dist} to account for vehicle rotation. 
If we know the maximum chassis yaw $\theta_{\max}$ during the manoeuvre,
we can compute an axis-aligned bounding rectangle as an outer approximation to the vehicle footprint.
This is useful for safety analysis, and is illustrated in Figure~\ref{fig:outer_approx}.

\begin{figure}[thpb]
  \centering
  \subfloat[\label{fig:outer_approx}]{\includegraphics[scale=0.4]{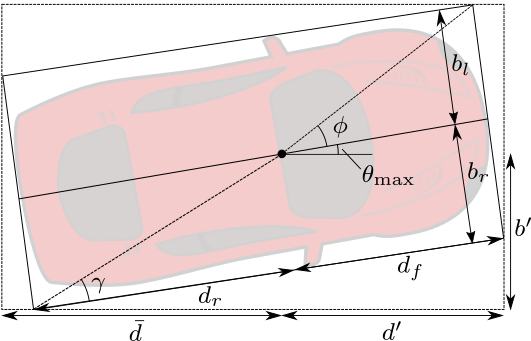}}\qquad
  \subfloat[\label{fig:inner_approx}]{\includegraphics[scale=0.4]{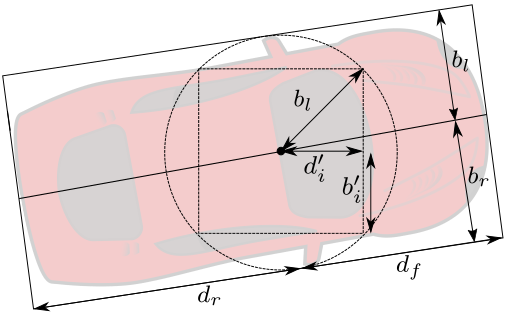}}\\
  \caption{(a) An outer approximation to a vehicle chassis that rotates by $\theta_{\max}$. The distances
  $d'$ and $\bar{d}$
  are used for longitudinal buffers during swerve manoeuvres, and $b'$ is used as a lateral
  buffer.
  (b) An inner approximation to a rotating vehicle chassis.
  }
\end{figure}

The three distances we need for safety analysis are from the centre of mass to the front
of the bounding rectangle, $d'$, from the centre of mass to the
side of the bounding rectangle, $b'$, and from the centre of mass to the rear of
the bounding rectangle, $\bar{d}$. 
The distances from the centre of mass
to the rear and front of the chassis are $d_r$ and $d_f$, respectively. 
The distances to the
left and right of the chassis are $b_l$ and $b_r$, respectively. 
As the vehicle rotates,
the length and width of the bounding rectangle increases until $\theta_{\max}$ reaches the angles 
from the centre of mass to the corners of the rectangle. Further rotation past these points decreases
the dimensions of the bounding rectangle. We can write these angles in terms of $\phi$ and $\gamma$,
illustrated in Figure~\ref{fig:outer_approx}.
The equations for the bounding rectangle distances are then
$d'$, $\bar{d}$, and $b'$ are
\begin{equation} 
\label{eq:d_outer}
d' = 
  \begin{cases} 
    d_f\cos(\theta_{\max}) + b_r\sin(\theta_{\max}) & \theta_{\max} \leq \phi, \\
    \sqrt{d_f^2 + b_r^2} & \theta_{\max} > \phi,
  \end{cases}
\end{equation}
\begin{equation}
\label{eq:d_back_outer}
\bar{d} = 
  \begin{cases} 
    d_r\cos(\theta_{\max}) + b_l\sin(\theta_{\max}) & \theta_{\max} \leq \gamma, \\
    \sqrt{d_r^2 + b_l^2} & \theta_{\max} > \gamma,
  \end{cases}
\end{equation}
\begin{equation}
\label{eq:b_outer}
b' =
  \begin{cases} 
    d_r\sin(\theta_{\max}) + b_r\cos(\theta_{\max}) & \theta_{\max}  \leq \frac{\pi}{2} - \gamma, \\
    \sqrt{d_r^2 + b_r^2} & \theta_{\max}  > \frac{\pi}{2} - \gamma.
  \end{cases}
\end{equation}
We now have an expression for the bounding rectangle distances of a rotating vehicle in terms of $\theta_{\max}$,
which is computed in Section~\ref{sec:method_swerve_for_brake}.

Using $b'$ and the lateral safe distance $d_{\text{lat}}$, we can now compute the lateral clearance distance, $y_c$ required for Subproblem 1.
\begin{equation}
\label{eq:lateral_clearance}
y_c = b' + b_l + d_{\text{lat}}.
\end{equation}
Let us denote the time $y_c$ is attained as $t_c$. 

\begin{theorem}
\Cref{eq:lateral_clearance} gives a lateral clearance distance sufficient for lateral safety when a swerving vehicle becomes laterally adjacent to another braking vehicle, or any time before.
\end{theorem}
\begin{proof}
To show lateral safety, we must show that laterally adjacent vehicles are at least $d_{\text{lat}}$
from one another, as given in \Cref{eq:rss_lat_dist}.
Since the swerving vehicle's lateral
speed is variable but nonnegative, a conservative lower bound
on its lateral velocity is zero when computing $d_{\text{lat}}$.
From assumption~\ref{swerve_brake_assumption}, since the other vehicle is braking, it is not swerving, and
therefore has zero lateral velocity during the swerve.
The required $d_{\text{lat}}$ can then be computed using \Cref{eq:rss_lat_dist}, taking $v_{r}^{\text{lat}}$ and $v_{f}^{\text{lat}}$ to be zero, and using the parameters $a^{\text{lat}}_{\min}$, $a^{\text{lat}}_{\max}$, and  $\rho$. The distance $d_{\text{lat}}$ acts as a buffer to ensure that upon reaching lateral adjacency, both agents are laterally safe from one another.

For $t < t_c$, the swerving vehicle is not laterally adjacent to the other vehicle, and
is laterally safe. For $t \geq t_c$, from Assumption~\ref{reasonable_manoeuvre_assumption}, 
$t_c$ is the time at which the two vehicles are closest while laterally adjacent.
From Equation~\ref{eq:lateral_clearance}, there is at least 
$d_{\text{lat}}$ of distance between the vehicles, and thus they are laterally
safe $\forall t \geq t_c$.
\end{proof}

\subsection{Swerving for a Braking Vehicle}
\label{sec:method_swerve_for_brake}

We can now use $y_c$ to compute the longitudinal safe distance, $d_{\textit{s,b}}$, required when swerving 
to avoid a braking lead vehicle. We wish to do
so under the constraints of the bicycle model outlined in Section~\ref{sec:preliminaries}.
In addition, if $\alpha$ denotes the lane width, $t_f$ denotes the end time of the swerve, and the
origin of the coordinate frame is at the center line of the current lane at the rear vehicle's position at $t=0$, we would like the swerve to satisfy the following boundary conditions:
\begin{equation}
\label{eq:bcs}
\theta(t_f) = 0,~y(t_f) = \alpha.
\end{equation}
However, to compute the optimal swerve manoeuvre with respect to longitudinal clearance 
is an optimization problem with no closed form solution~\cite{shiller_sundar_1996}. 
Instead, we can compute a swerve manoeuvre feasible for the bicycle model, 
and use that to obtain an upper bound on the 
actual longitudinal distance required by a swerve constrained by the bicycle model.

As in \Cref{eq:rss_long_dist}, the lead vehicle is travelling with velocity $v_f$, and
then brakes at $a_{\max,\text{brake}}$ during the entire
manoeuvre. The swerve is preceded by the rear vehicle maximally accelerating during
the reaction delay $\rho$, at which point it begins the swerve manoeuvre with initial
speed $v_{r,\rho}$. To ensure monotonicity in the gap between the rear and lead vehicles, 
a lower bound on the distance travelled until $t_f$ by the lead vehicle is used, 
denoted $x_f$.

The swerve we consider is bang-bang in the steering input with zero longitudinal acceleration, and is illustrated
in Figure~\ref{fig:swerve_manoeuvre}.
We denote the longitudinal distance travelled by the swerving vehicle until the swerving vehicle reaches the lateral clearance distance as $x_c$ This distance $x_c$ is computed in Equations~\ref{eq:x_c_1} and~\ref{eq:x_c_2}.

For the swerve manoeuvre, the turning radius of the circular arcs depends on the maximum lateral acceleration, as well as the kinematic limits
of the steering angle. 
The constraints on steering angle and lateral acceleration from (\ref{eq:bicycle_model})
give two constraints on the turning radius
\begin{align}
R_{\min,\delta} = \sqrt{\frac{(l_r + l_f)^2}{\tan(\delta_{\max})^2 + l_r^2}}, &&
R_{\min,a} = \frac{v_{r,\rho}^2}{a^{\text{lat}}_{\min}}.
\end{align}
To ensure both constraints are satisfied, we set $R_c$ from (\ref{eq:bicycle_model}) to the maximum of the two.
From this turning radius, we can compute the steering angle $\delta_c$ and the
slip angle $\beta_c$
\begin{align}
\delta_c = \tan^{-1}\left(\sqrt{\frac{(l_r+l_f)^2}{R_c^2 - l_r^2}}\right), &&
\beta_c = \tan^{-1}\left(\frac{l_r\tan(\delta_c)}{l_r+l_f}\right). 
\end{align} 

\begin{figure}[thpb]
  \centering
  \includegraphics[scale=0.33]{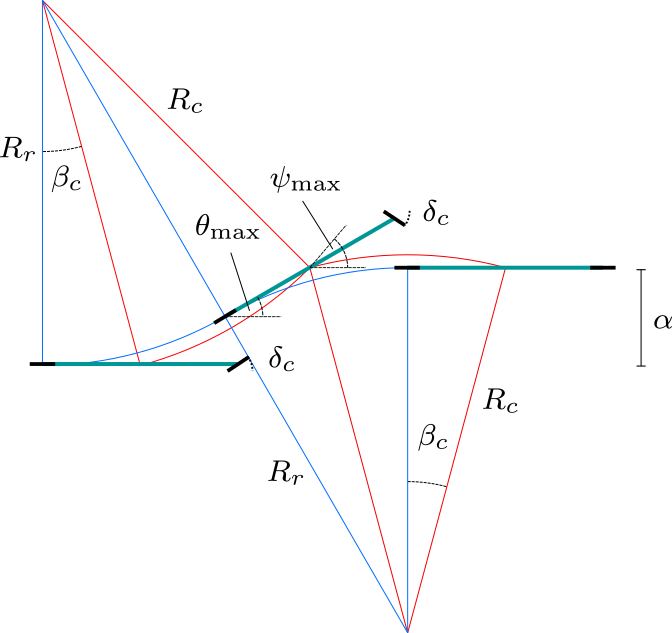}
  \caption{
    \label{fig:swerve_manoeuvre}
    The swerve manoeuvre used for safety analysis. The red path is taken by the
    centre of mass, and the blue path is taken by the rear axle. The distance between
    lanes is $\alpha$, $\delta_c$ is the steering angle, $\beta_c$ is the slip angle. The maximum angles
    achieved by the chassis yaw and the velocity of the centre of mass are given by $\theta_{\max}$ and $\psi_{\max}$,
    respectively. The turning radius of the rear axle and centre of mass's paths are given by
    $R_r$ and $R_c$, respectively.
  }
\end{figure}

We can now compute the $\theta_{\max}$ required to satisfy 
the boundary conditions in ~\Cref{eq:bcs}. From the rear axle, the two circular arcs are 
symmetrical
in lateral distance travelled, as in Figure~\ref{fig:swerve_manoeuvre}. 
Therefore, we can compute the angle along the first circular 
arc required to reach a lateral distance of $\frac{\alpha}{2}$. 
First, we compute the turning radius at the rear axle, $R_r$
\begin{equation}
\label{eq:rear_axle_radius}
R_r = \frac{l_r + l_f}{\tan(\delta_c)}.
\end{equation}
The lateral distance travelled during the first circular arc is then given by
\begin{equation}
\label{eq:y_first_arc}
y(t) = R_r(1 - \cos(\theta(t))).
\end{equation}
For a given value of $\delta_c$, $\theta_{\max}$ is then
\begin{equation}
\label{eq:theta_max}
\theta_{\max} = \cos^{-1}\left(1 - \frac{\alpha}{2R_r}\right).
\end{equation}

To compute $x_c$, 
there are two cases, depending on if $y_c$ is reached in the first or second circular arc.
We can compute $\psi_{\max}$ using (\ref{eq:bicycle_model}).
From Assumption~\ref{forward_motion_assumption}, we have that $\psi_{\max} \leq \frac{\pi}{2}$. 
Thus, the first case occurs if
\begin{equation}
y_c \leq R_c(\cos(\beta_c)-\cos(\psi_{\max})),
\end{equation}
otherwise the second case occurs. 

\subsubsection{First Circular Arc}
Similar to~\Cref{eq:y_first_arc}, the longitudinal position along the first circular arc
is given by
\begin{equation}
\label{eq:x_first_arc}
x(t) = R_c(\sin(\psi(t)) - \sin(\beta_c)).
\end{equation}
We can use the centre of mass equivalent of \Cref{eq:y_first_arc} and $y_c$ to compute the 
$\psi$ value at the clearance point, $\psi_c$ 
\begin{equation}
\psi_c = \cos^{-1}\left(\cos(\beta_c)-\frac{y_c}{R_c}\right).
\end{equation}
Substituting this value for $\psi$ in~\Cref{eq:x_first_arc} gives our swerve longitudinal distance
\begin{equation}
\label{eq:x_c_1}
x_c = R_c(\sin(\psi_c)-\sin(\beta_c)).
\end{equation}
The magnitude of the velocity is constant during the swerve, and so we can compute $t_c$
using the arc length travelled up to the clearance point $y_c$,
\begin{equation}
\label{eq:t_c_1}
t_c = \frac{R_c(\psi_c - \beta_c)}{v}.
\end{equation}

\subsubsection{Second Circular Arc}
In the second circular arc, we denote the initial heading of the centre of mass as  $\hat{\psi} = \psi_{\max} - 2\beta_c$,
the initial $x$ position as $\hat{x} = R_c(\sin(\psi_{\max}) - \sin(\beta_c))$, and the initial $y$ position as $\hat{y} = R_c(\cos(\beta_c)-\cos(\psi_{\max})$.
The longitudinal and lateral distance along this arc are then
\begin{equation}
\label{eq:x_second_arc}
x(t) = R_c(\sin(\hat{\psi}) - \sin(\psi(t))) + \hat{x},
\end{equation}
\begin{equation}
y(t) = R_c(\cos(\psi(t)) - \cos(\hat{\psi})) + \hat{y}.
\end{equation}
As in Case 1, substituting $y_c$ gives us $\psi_c$,
\begin{equation}
\psi_c = \cos^{-1}\left(\frac{1}{R_c}\left(y_c - \hat{y}\right) + \cos(\hat{\psi})\right)
\end{equation}
Substituting this value for $\psi$ in~\Cref{eq:x_second_arc} and add $d'$ gives
\begin{equation}
\label{eq:x_c_2}
x_c = R_c(\sin(\hat{\psi}) - \sin(\psi_c)) + \hat{x} + d'.
\end{equation}
Similar to Case 1, we can then compute the clearance time $t_c$,
\begin{equation}
\label{eq:t_c_2}
t_c = \frac{R_c(\psi_{\max} - \beta_c + \hat{\psi} - \psi_c)}{v}.
\end{equation}

From these longitudinal swerve clearance values, we can then compute the longitudinal
safe distance. To do this, we can replace the
rear braking distance in~\Cref{eq:rss_long_dist} with the longitudinal swerve distance $x_c$. In addition,
to ensure a monotonically decreasing gap between the two vehicles, we set the initial
speed of the lead vehicle (as a conservative approximation) to 
\begin{equation}
\label{eq:lead_vehicle_speed_bound}
v_f' = \min(v_f, v_r\cos(\psi_{\max})).
\end{equation}
The braking distance of the lead vehicle occurs during the reaction time $\rho$ and the swerve
clearance time $t_c$, giving a front vehicle braking distance of
\begin{equation}
x_f = v_f'(\rho+t_c) - \frac{a_{\max,\text{brake}}(\rho+t_c)^2}{2}.
\end{equation}

Using the parameters $a_{\max,\text{accel}}, \rho$ introduced in Section~\ref{sec:preliminaries},
the longitudinal safe distance between a swerving rear vehicle and a braking lead vehicle is 
\begin{equation}
\label{eq:d_long_swerve_for_brake}
d_{\textit{s,b}} = \biggl[ v_r \rho + \frac{1}{2} a_{\max,\text{accel}}\rho^2 + x_c - x_f \biggr]_+ + d' + d.
\end{equation}

\begin{theorem}
\Cref{eq:d_long_swerve_for_brake} gives a longitudinal safe distance
sufficient for safety when swerving for a braking lead vehicle.
\end{theorem}
\begin{proof}

For $t > t_c$, $y(t) > y_c$, and therefore the swerving vehicle is no
longer longitudinally adjacent to the lead vehicle, so is safe from the lead vehicle's braking.
For $t \leq t_c$, from~\Cref{eq:lead_vehicle_speed_bound}, we use a conservative lower bound 
for the speed of the lead vehicle to ensure the lead vehicle's speed
is less than the swerving vehicle during the entire swerve. This implies the gap between the 
two vehicles is monotonically decreasing. This means the minimum gap between the two vehicles 
occurs at time $t_c$. 

The swerving vehicle travels $x_c + v_r\rho+\frac{1}{2}a_{\max,\text{accel}}\rho^2$,
and a conservative lower bound on the lead vehicle's travel distance is 
$v_f'(\rho+t_c) - \frac{1}{2}a_{\max,\text{brake}}(\rho+t_c)^2$. There is at most $d'$ of distance from
the centre of mass to the front of the swerving vehicle. 
Thus, if a swerving vehicle
maintains distance $d_{\textit{s,b}}$, it is safe from the lead vehicle at time $t_c$.
Since the gap is monotonically decreasing for $t \leq t_c$, it is safe
$\forall t \leq t_c$.

\end{proof}

\subsection{Braking for a Swerving Vehicle}
\label{sec:method_braking_for_swerving_vehicle}

The longitudinal safe distance required to swerve for a braking vehicle was computed in the preceding section, 
and this section considers the opposite problem, computing the longitudinal safe distance required to brake 
for a swerving lead vehicle without collision. Since the lead vehicle intends
to occupy the other lane, it requires less longitudinal distance for the rear vehicle to brake
to avoid the swerving lead vehicle than it would for it to brake for a braking lead vehicle. 
It is assumed the front vehicle is performing the same swerve discussed in Section~\ref{sec:method_swerve_for_brake}.
To account for rotation of the front vehicle, $\bar{d}$ is used to compensate as defined in Section~\ref{sec:method_lat_clear}.

Equations~\ref{eq:x_c_1}, \ref{eq:t_c_1}, \ref{eq:x_c_1}, and \ref{eq:t_c_2} can be used to compute
the $x_c$ and $t_c$ for the front vehicle's swerve.
As in Equation~\ref{eq:rss_long_dist}, it is assumed that the rear vehicle accelerates maximally during its 
reaction time, and then brakes comfortably until $t_c$. 
As before, denote the rear vehicle's post-acceleration velocity
as $v_{r,\rho}$. Then its minimum velocity during the braking manoeuvre is
\begin{equation}
v_{r,\min} = \max(\min(v_r, v_{r,\rho} - a_{\min,\text{brake}}(t_c-\rho)), 0).
\end{equation}
As in Section~\ref{sec:method_swerve_for_brake}, the proof of safety is simplified if the gap 
is monotonically decreasing until lateral safety is reached. To ensure this, the lead vehicle speed is 
conservatively bounded with $v_f'$
\begin{equation}
\label{eq:v2_swerve_lower_bound}
v_f' = \min(v_f\cos(\psi_{\max}), v_{r,\min}).
\end{equation}
A conservative lower bound for the longitudinal distance travelled by the swerving front vehicle is then
\begin{equation}
\label{eq:x_f_b_s}
x_f = v_f't_c.
\end{equation}
The distance $x_f$ is a lower bound on the distance travelled by the front vehicle during the 
swerve that creates a monotonically decreasing gap. 

The distance travelled by the rear braking vehicle during its reactions delay and its braking manoeuvre is denoted by $x_r$.
This distance depends on the clearance time $t_c$, similar to 
the distance travelled by the front vehicle in the preceding section. The distance travelled
during the rear vehicle's braking manoeuvre, $x_{r,\text{brake}}$, is given by 
\begin{equation}
x_{r,\text{brake}} = 
  \begin{cases} 
    v_{r,\rho}(t_c - \rho) - \frac{a_{\min,\text{brake}}(t_c - \rho)^2}{2}, & t_c - \rho \leq \frac{v_{r,\rho}}{a_{\min,\text{brake}}},\\
    \frac{v_{r,\rho}^2}{2a_{\min,\text{brake}}}, & t_c - \rho > \frac{v_{r,\rho}}{a_{\min,\text{brake}}}.
  \end{cases}
\end{equation}

Following this, the distance travelled by the braking rear vehicle is
\begin{equation}
\label{eq:x_r_b_s}
x_r = \frac{(v_r + v_{r,\rho})\rho}{2} + x_{r,\text{brake}}.
\end{equation}

Using Equations \ref{eq:x_f_b_s} and \ref{eq:x_r_b_s}, the longitudinal safe distance 
when braking for a swerving vehicle, $d_{\textit{b,s}}$ is then
\begin{equation}
\label{eq:d_long_brake_for_swerve}
d_{\textit{b,s}} = \left[x_r - x_f\right]_+ + d_f + \bar{d}.
\end{equation}

\begin{theorem}
Equation~\ref{eq:d_long_brake_for_swerve} gives a longitudinal safe distance
sufficient for safety when braking for a swerving lead vehicle.
\end{theorem}
\begin{proof}
For $t > t_c$, the swerving vehicle is laterally clear from the rear braking vehicle, and therefore
the rear vehicle is safe. The velocity used for the lead vehicle is a conservative lower bound
on its true speed $\forall t \leq t_c$, as per Equation~\ref{eq:v2_swerve_lower_bound}. 
In addition, $v_f' \leq v_r$, $\forall t \leq t_c$,
and as a result the gap between the two vehicles is monotonically
decreasing on that interval. The minimum distance between the two vehicles thus occurs
at time $t_c$. Equation~\ref{eq:d_long_brake_for_swerve} thus gives enough
clearance such that no collision occurs at time $t_c$, so the rear vehicle is safe at time $t_c$. Since the
gap is monotonically decreasing over the interval, the rear vehicle is safe $\forall t \leq t_c$.
\end{proof}

\subsection{Swerving for a Swerving vehicle}
\label{sec:method_swerve_for_swerve}

The final relevant longitudinal safe distance is the distance required when swerving behind a swerving
lead vehicle. This is illustrated in Figure~\ref{fig:swerve_for_swerve_example}. Both vehicles are longitudinally
adjacent during the entire manoeuvre. From Assumption~\ref{swerve_brake_assumption}, the lead vehicle will
not brake during its swerve. The goal is then to compute the longitudinal distance required to swerve
behind a lead swerving vehicle, such that if the lead vehicle were to immediately brake with deceleration 
$a_{\max, \text{accel}}$ at the end of its swerve, and the rear vehicle were to brake with deceleration
$a_{\min, \text{accel}}$ at the end of its reaction-delayed swerve, there would be no collision.
The swerve completion times of the rear and front vehicle are given by $t_1$ and $t_2$, respectively.
Similar to the previous section, $v_f'$ denotes a conservative lower bound on the front vehicle's speed.
The longitudinal safe distance required to swerve in response to a swerving vehicle, $d_{\textit{s,s}}$, is then
\begin{multline}
\label{eq:d_long_swerve_swerve}
d_{\textit{s,s}} = \frac{v_r + v_{r,\rho}}{2}\rho + v_{r,\rho}(t_1 - \rho) + \\ \frac{v_{r,\rho}^2}{2a_{\min,\text{brake}}} - \left(v_f't_2 + \frac{v_f'^2}{2a_{\max,\text{brake}}}\right) + d' + \bar{d}. 
\end{multline}

\begin{theorem}
Equation~\ref{eq:d_long_swerve_swerve} gives a longitudinal safe distance sufficient for 
safety when swerving for a swerving lead vehicle.
\end{theorem}
\begin{proof}
The gap between each vehicle can be written as a piecewise function of time. The endpoints of the intervals
are functions of the reaction delay, $\rho$, the duration of the front vehicle's swerve, $t_2$,
the duration of the rear vehicle's swerve, $t_1$, the brake time of the front vehicle, $t_{b,2}$,
and the brake time of the rear vehicle, $t_{b,1}$. The swerve times for the kinematic bicycle model for varying speeds 
are proportional to $v\cos^{-1}\left(1 - \frac{1}{v^2}\right)$, which is quasi-constant across all relevant 
road speeds. In addition, $a_{\max, \text{accel}} > a_{\min, \text{accel}}$, and swerve times are longer than
reasonable reaction times. From this, it is reasonable to assume that the interval endpoints are
$\rho < t_2 < \rho + t_1 < t_2 + t_{b,2} < \rho + t_1 + t_{b,1}$.
Denote the longitudinal distance travelled during the swerves by the front and rear vehicle as
$x_{s,2}(t)$ and $x_{s,1}(t)$ respectively, the initial gap between the vehicles by $g_0$, and the gap between the vehicles as $g(t)$.

%\begin{strip}
%\begin{equation} 
%\label{eq:g_swerve_swerve}
%g(t) = 
%  \begin{cases} 
%     g_0 + x_{s,2}(t) - (v_rt + \frac{1}{2}a_{\max,\text{accel}}t^2) & t \leq \rho, \\
%     g_0 + x_{s,2}(t) - (\frac{v_r + v_{r,\rho}}{2}\rho + x_{s,1}(t - \rho)) & \rho < t \leq t_2, \\
%     g_0 + x_{s,2}(t_2) + v_f(t - t_2) - \frac{1}{2}a_{\max,\text{brake}}(t-t_2)^2 - \\
%     ~~~~(\frac{v_r + v_{r,\rho}}{2}\rho + x_{s,1}(t - \rho)) & t_2 < t \leq \rho + t_1, \\
%     g_0 + x_{s,2}(t_2) + v_f(t - t_2) - \frac{1}{2}a_{\max,\text{brake}}(t-t_2)^2 - \\
%     ~~~~(\frac{v_r + v_{r,\rho}}{2}\rho + x_{s,1}(t_1 - \rho) + v_{r,\rho}(t - t_1) - \frac{1}{2}a_{\min,\text{brake}}(t-t_1)^2)  & \rho + t_1 < t \leq t_2 + t_{b,2}, \\
%     g_0 + x_{s,2}(t_2) + \frac{v_f^2}{2a_{\max,\text{brake}}} - (\frac{v_r + v_{r,\rho}}{2}\rho + x_{s,1}(t_1 - \rho) + \\
%     ~~~~v_{r,\rho}(t - t_1) - \frac{1}{2}a_{\min,\text{brake}}(t-t_1)^2)  & t_2 + t_{b,2} < t \leq \rho + t_1 + t_{b,1}, \\
%     g_0 + x_{s,2}(t_2) + \frac{v_f^2}{2a_{\max,\text{brake}}} - (\frac{v_r + v_{r,\rho}}{2}\rho + x_{s,1}(t_1 - \rho) + \frac{v_{r,\rho}^2}{2a_{\min,\text{brake}}}) & t > \rho + t_1 + t_{b,1}.
%  \end{cases}
%\end{equation}
%\end{strip}

The maximum longitudinal velocity during the rear vehicle swerve is $v_{r,\rho}$. If the maximum
$\psi$ value during the front vehicles swerve is denoted $\psi_{\max,f}$, the minimum
longitudinal velocity during the front vehicle's swerve is given by $v_f\cos(\psi_{\max,f})$.
Set $v_f' = \min(v_f\cos(\psi_{\max,f}), v_r)$.
This means that
\begin{equation}
\label{eq:x_s_1}
x_{s,1}(t) \leq v_{r,\rho} t,
\end{equation}
\begin{equation}
\label{eq:x_s_2}
x_{s,2}(t) \geq v_f' t.
\end{equation}

Using \Cref{eq:x_s_1,eq:x_s_2} as conservative bounds on the distance travelled by both vehicles during the swerve manoeuvre results in a monotonically decreasing function of $t$, $\hat{g}(t)$,
with the property that $\hat{g}(t) \leq g(t), \forall t$.
%\begin{strip}
%\begin{equation} 
%\label{eq:g_hat_swerve_swerve}
%\hat{g}(t) = 
%  \begin{cases} 
%     g_0 + v_f't - (v_rt + \frac{1}{2}a_{\max,\text{accel}}t^2) & t \leq \rho, \\
%     g_0 + v_f't - (\frac{v_r + v_{r,\rho}}{2}\rho + v_{r,\rho}(t - \rho)) & \rho < t \leq t_2, \\
%     g_0 + v_f't - \frac{1}{2}a_{\max,\text{brake}}(t-t_2)^2 - (\frac{v_r + v_{r,\rho}}{2}\rho + v_{r,\rho}(t - \rho)) & t_2 < t \leq \rho + t_1, \\
%     g_0 + v_f't - \frac{1}{2}a_{\max,\text{brake}}(t-t_2)^2 - \\
%     ~~~~(\frac{v_r + v_{r,\rho}}{2}\rho + v_{r,\rho}(t - \rho) - \frac{1}{2}a_{\min,\text{brake}}(t-t_1)^2)  & \rho + t_1 < t \leq t_2 + t_{b,2}, \\
%     g_0 + v_f't_2 + \frac{v_f'^2}{2a_{\max,\text{brake}}} - (\frac{v_r + v_{r,\rho}}{2}\rho + v_{r,\rho}(t - \rho) - \\
%     ~~~~\frac{1}{2}a_{\min,\text{brake}}(t - t_1)^2)  & t_2 + t_{b,2} < t \leq \rho + t_1 + t_{b,1}, \\
%     g_0 + v_f't_2 + \frac{v_f'^2}{2a_{\max,\text{brake}}} - (\frac{v_r + v_{r,\rho}}{2}\rho + v_{r,\rho}(t_1 - \rho) + \frac{v_{r,\rho}^2}{2a_{\min,\text{brake}}})  & t > \rho + t_1 + t_{b,1}.
%  \end{cases}
%\end{equation}
%\end{strip}

This implies that the minimum of $\hat{g}(t)$ occurs for $t > t_{b,1}$, where $\hat{g}(t)$ is constant
\begin{multline}
\min_t \hat{g}(t) = g_0 + v_f't_2 + \frac{v_f'^2}{2a_{\max,\text{brake}}} -\\ 
~~~~\left(\frac{v_r + v_{r,\rho}}{2}\rho +
v_{r,\rho}t_1 + \frac{v_{r,\rho}^2}{2a_{\min,\text{brake}}}\right).
\end{multline}
Since $\hat{g}(t) \leq g(t), \forall t$, if $\hat{g}(t) \geq 0, \forall t$, no collision occurs.
This is satisfied if the initial gap satisfies
\begin{multline}
g_0 \geq \frac{v_r + v_{r,\rho}}{2}\rho + v_{r,\rho}t_1 + \frac{v_{r,\rho}^2}{2a_{\min,\text{brake}}} - \\
~~~~\left(v_f't_2 + \frac{v_f'^2}{2a_{\max,\text{brake}}}\right).
\end{multline}
By adding in the distances from the centre of mass to the ends of the chassis, compensating for the rotation
of each swerving vehicle, an initial gap is sufficient for safety $\forall t$ if
\begin{multline}
g_0 \geq \frac{v_r + v_{r,\rho}}{2}\rho + v_{r,\rho}t_1 + \frac{v_{r,\rho}^2}{2a_{\min,\text{brake}}} - \\
~~~~\left(v_f't_2 + \frac{v_f'^2}{2a_{\max,\text{brake}}}\right) + d' + \bar{d}.
\end{multline}
Which yields Equation~\ref{eq:d_long_swerve_swerve}.

At $t \geq t_2$, the time at which the lead vehicle begins hard braking, there is enough
longitudinal distance to brake for the leading vehicle, as $\hat{g}(t) \geq 0, \forall t \geq t_2$,
so the rear vehicle is safe. Since $\hat{g}(t)$ 
is monotonically decreasing with respect to $t$, the safe longitudinal distance is satisfied for 
$t < t_2$, and thus the rear vehicle is safe $\forall t$.
\end{proof}

\subsection{Universal Following Distance}
\label{sec:method_universal_following_distance}

The final subproblem addressed in this paper aims to combine the results of the
previous sections into a final following distance that can be maintained by all
vehicles in a given straight road system to ensure universal safety, assuming the vehicles
can brake or swerve as a response to the behaviour of other vehicles in front of them.
In this sense, this section extends the analysis of the preceding sections into the case
of more than two vehicles in a road system.
Each vehicle's following distance will be a function of the speed
of the vehicle, as well as the speed of the 2 vehicles in front of the vehicle, and the parameters
outlined in Section~\ref{sec:preliminaries}. 
Denote the distance required to brake for a braking lead vehicle as $d_{\textit{b,b}}(v_{r}, v_{f}, \rho)$,
the distance required to swerve for a braking lead vehicle as $d_{\textit{b,s}}(v_{r}, v_{f}, \rho)$,
the distance required to swerve for a braking lead vehicle as $d_{\textit{s,b}}(v_{r}, v_{f}, \rho)$,
and the distance required to swerve for a swerving lead vehicle as $d_{\textit{s,s}}(v_{r}, v_{f}, \rho)$.

\begin{figure}[thpb]
  \centering
  \subfloat[\label{fig:utopia_swerve}]{\includegraphics[scale=0.4]{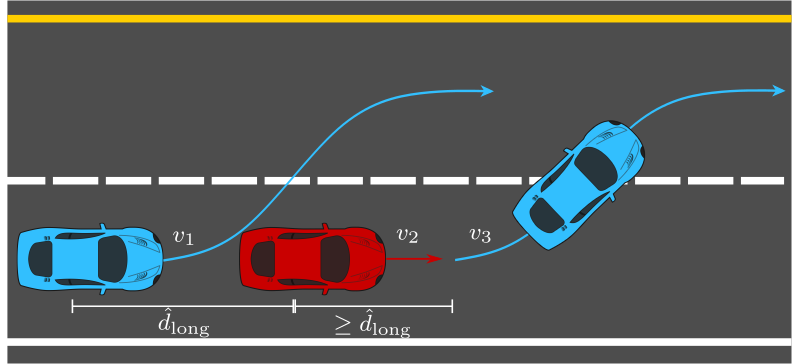}}\qquad
  \subfloat[\label{fig:utopia_brake}]{\includegraphics[scale=0.4]{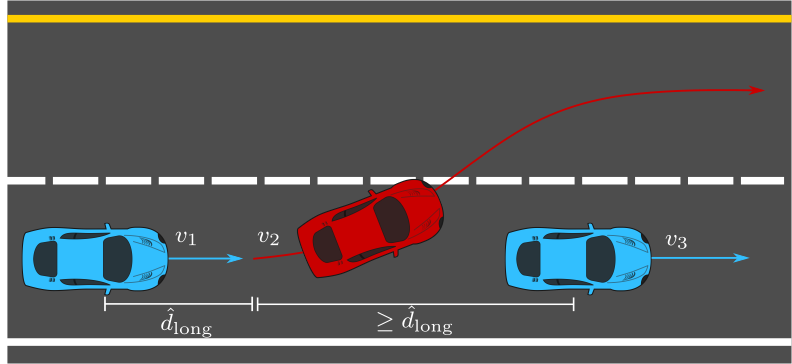}}\qquad
  \caption{
    \label{fig:utopia_brake_swerve}
    (a) Scenario where the rear vehicle must swerve for a swerving vehicle 2 cars ahead.
    (b) Scenario where the rear vehicle must brake for a braking vehicle 2 cars ahead.
  }
\end{figure}

In such a road system, there will be blocks of vehicles where the front vehicle in the block is much farther away from the nearest vehicle in front of it than both $d_{\textit{b,b}}$ 
and $d_{\textit{s,s}}$. Since it is at least this far, 
it can safely brake or swerve for any vehicle in front of it, and therefore any vehicle in front of it can be ignored.
Because of this, these blocks can be considered in isolation, and if each
block of vehicles is considered safe, then all vehicles in the road system are considered safe.
For any vehicle in a given block, denote its speed by $v_1$, and the speeds of the first and second
vehicles in front of it (if they exist within the block) as $v_2$ and $v_3$, respectively.
The longitudinal position of each vehicle as a function of time is denoted by $x_1(t)$, $x_2(t)$, and $x_3(t)$.
A sufficient safe following distance for each vehicle is then
\begin{multline}
\label{eq:d_long_universal}
\hat{d}_{\text{long}} = \max(d_{\textit{b,s}}(v_{1}, v_{2}, \rho), d_{\textit{b,s}}(v_{1}, v_{2}, \rho), \\
d_{\textit{s,s}}(v_{1}, v_{3}, 2\rho) - d_{\textit{s,b}}(v_{2}, v_{3}, \rho), \\
d_{\textit{b,b}}(v_{1}, v_{3}, 2\rho) - d_{\textit{s,b}}(v_{2}, v_{3}, \rho)).
\end{multline}

\begin{theorem}
Equation~\ref{eq:d_long_universal} gives a longitudinal safe distance sufficient for 
universal safety when maintained by all vehicles.
\end{theorem}
\begin{proof}
As mentioned earlier, each block of vehicles can be analyzed individually for safety, and if every block
is safe, all vehicles are safe. The safety of any given block can be proved using an inductive argument
across all of the vehicles, starting from the front of the block. 
The following is a proof sketch.
\begin{itemize}
	\item For the base case, the safety of the first two vehicles is proven
	when following with at least $\hat{d}_{\text{long}}$.

	\item For the inductive step, it is assumed the $i^{\text{th}}$ agent is following with
	at least $\hat{d}_{\text{long}}$ and is safe, and it is shown that if the 
  $(i+1)^{\text{th}}$ agent follows with at least $\hat{d}_{\text{long}}$, then it is safe.
\end{itemize}

\subsubsection{Base Case}
The first vehicle at the front of the block is by definition at least $d_{\textit{b,b}}$ and $d_{\textit{s,s}}$ 
from any vehicle in front of it (if such a vehicle exists). 
As a result, any potential vehicle in front of the first can be safely avoided
if necessary with either a brake or a swerve. This means that the first vehicle
in the block is safe, and any potential vehicle in front of the first
can be safely ignored by all vehicles in the block.

The second vehicle follows the first vehicle at 
$\hat{d}_{\text{long}}$. If the front vehicle brakes, the second vehicle is at least $d_{\textit{s,b}}$
away from it, and can swerve to safety. If the front vehicle swerves, the second vehicle is
at least $d_{\textit{b,s}}$ away from it, and can brake safely. The second vehicle will therefore not collide with the first vehicle, and is therefore safe.

\subsubsection{Induction}
Now, suppose the $i^{\text{th}}$ vehicle is following with at least $\hat{d}_{\text{long}}$
of distance, and is safe from the vehicles in front of it. Denote the $(i+1)^{\text{th}}$ as
vehicle 1, the $i^{\text{th}}$ vehicle as vehicle 2, and the $(i-1)^{\text{th}}$ vehicle as vehicle 3.
The distance between vehicle 1 and vehicle 2 is $\hat{d}_{\text{long}}$.
If vehicle 2 brakes or swerves, vehicle 1 is at least $d_{\textit{s,b}}$ and $d_{\textit{b,s}}$ away
from it, and is safe from vehicle 2 if it responds with a swerve or brake, respectively.

If vehicle 1 swerves in response to vehicle 2's brake, there are 2 cases to consider.
The first case is if vehicle 3 was braking.
Since vehicle 2 was assumed to be safe from vehicle 3, $x_2(t) \leq x_3(t), \forall t$.
Combining this with the fact that $d_{\textit{s,b}}$ is sufficient for vehicle 1 to swerve safely from
vehicle 2, vehicle 1 must be safe from vehicle 3 if vehicle 3 brakes.

If vehicle 3 was swerving, $d_{\textit{s,s}}(v_1, v_3, 2\rho)$ is a sufficient distance for vehicle 1 to follow
vehicle 3 to ensure safety.
This case is illustrated in Figure~\ref{fig:utopia_swerve}.
The reaction delay is doubled to account for the reaction propagating through 2 vehicles instead
of the usual one. Since vehicle 2 was assumed to be safe from vehicle 3, $d_{\textit{s,b}}(v_2, v_3, \rho)$ is
a lower bound on vehicle 2's following distance from vehicle 3. This means that in this case, 
$d_{\textit{s,s}}(v_{1}, v_{3}, 2\rho) - d_{\textit{s,b}}(v_{2}, v_{3}, \rho)$ is a sufficient
following distance between vehicle 1 and 2 to guarantee safety.

If vehicle 1 brakes in response to vehicle 2's swerve, as before there are 2 cases to consider.
The first case is if vehicle 3 was swerving. As before, since vehicle 2 was assumed to be safe
from vehicle 3, $x_2(t) \leq x_3(t), \forall t$.
Combining this with the fact that $d_{\textit{b,s}}$ is sufficient for vehicle 1 to brake safely from
vehicle 2's swerve, vehicle 1 must be safe from vehicle 3's swerve.

If vehicle 3 was braking, $d_{\textit{b,b}}(v_1, v_3, 2\rho)$ is a sufficient distance for vehicle 1 to follow
vehicle 3 to ensure safety.
This case is illustrated in Figure~\ref{fig:utopia_brake}.
Again, the reaction delay is doubled to account for propagation between two vehicles. Since vehicle 2
was assumed to be safe from vehicle 3, $d_{\textit{s,b}}(v_2, v_3, \rho)$ is again a lower bound on vehicle 2's
following distance. Thus, in this case, 
$d_{\textit{b,b}}(v_{1}, v_{3}, 2\rho) - d_{\textit{s,b}}(v_{2}, v_{3}, \rho)$ is a sufficient
following distance between vehicle 1 and 2 to guarantee safety.

Since $\hat{d}_{\text{long}}$ is greater or equal to each of these following distances, vehicle 1
is safe, and thus the $(i+1)^{\text{th}}$ is safe. By induction, any block of vehicles where each vehicle maintains the following distance given in \Cref{eq:d_long_universal} is safe, and as a result, the entire system is safe.
\end{proof}

At high speeds, this new following distance can be used to allow for tighter following between agents. At low speeds, the agents can revert to the braking following distance used in RSS. 

If the positions of vehicles 2 and 3, $d_2$ and $d_3$ respectively, are
known as well as their speeds, the following distance can be improved 
further. If we denote $d_{2,3} = d_3 - d_2$, then by the same logic in 
the preceding proof, a sufficient longitudinal safe distance is
\begin{multline}
\hat{d}_{\text{long}} = \max(d_{\textit{b,s}}(v_{1}, v_{2}, \rho), d_{\textit{b,s}}(v_{1}, v_{2}, \rho), \\
d_{\textit{s,s}}(v_{1}, v_{3}, 2\rho) - d_{2,3},
d_{\textit{b,b}}(v_{1}, v_{3}, 2\rho) - d_{2,3}).
\end{multline}

A comparison between the RSS following distance and these new following
distances across a range of speeds is shown in
Figure~\ref{fig:d_safe_combined}. In the plot, all vehicles are moving at the same speed.
Since $d_{2,3}$ can vary, for illustration 
purposes we assume each vehicle follows the agent in front of it at 
\begin{multline}
\hat{d}_{\text{long}} = \max(d_{\textit{b,s}}(v_{1}, v_{2}, \rho), d_{\textit{b,s}}(v_{1}, v_{2}, \rho), \\
\frac{d_{\textit{s,s}}(v_{1}, v_{3}, 2\rho)}{2},
\frac{d_{\textit{b,b}}(v_{1}, v_{3}, 2\rho)}{2}).
\end{multline}
This results in a uniform following distance across all agents when they are moving at the same speed.

\section{Validation and Results}
\label{sec:validation_and_results}

To validate our bicycle model assumptions, we first check the validity of our conservative
upper bound on the required swerve distance by computing a lower bound. 
We then use a dynamic vehicle 
model~\cite{gerdts_2005} to see if our computed swerve distances are reasonable approximations. 
The lower bound is computed and compared to the upper bound distance, as well as the relevant
braking distance, in Section~\ref{sec:swerve_lower_bound}. In 
Section~\ref{sec:swerve_experimental_setup}, we compare our swerve clearance distance, as
computed in Section~\ref{sec:method_swerve_for_brake}, to swerves from the dynamic
model.

\subsection{Lower Bound Validation}
\label{sec:swerve_lower_bound}

To compute a lower bound on the longitudinal swerve clearance distance $x_c$, we
use the particle model in \Cref{eq:pm}. We
set the minimum $a_x$ and maximum $a_y$ values to be $-a_{\min,\text{brake}}$ and $a_{\min}^{\text{lat}}$, respectively,  
from the bicycle model. This ensures that any acceleration possible under the bicycle 
model is also possible under the particle model.

For a particle model, maximal lateral acceleration towards $y_c$ as well as maximal longitudinal deceleration leads to lateral clearance in the shortest longitudinal distance $\bar{x}_c$~\cite{shiller_sundar_1998}. Thus, we have that $\bar{x}_c \leq x_c$ for any other manoeuvre feasible for the particle model.

Finally, for computing the clearance, we use an inner approximation
of the vehicle's chassis during rotation. To do so, we use the square inscribed
on the circle of radius $b_l$ centred on the centre of mass with side length $2d_i'$. 
This is shown in Figure~\ref{fig:inner_approx}. Through this inner approximation, we have that $d_i' \leq d'$ for any possible chassis rotation. This implies
that anything the chassis can clear during the swerve will be cleared by the 
inner square. 
If we use $x_f$ as in Section~\ref{sec:method_swerve_for_brake}, a lower bound on the longitudinal safe distance, denoted by $\bar{d}_{\text{long}}$, is given by
\begin{equation}
\label{eq:lower_bound_swerve}
\bar{d}_{\text{long}} = v_r \rho + \frac{1}{2}a_{\max,\text{accel}}\rho^2 + \bar{x}_c - x_f.  
\end{equation}
\begin{theorem}
\Cref{eq:lower_bound_swerve} gives a longitudinal safe distance
necessary for safety when swerving for a braking lead vehicle.
\end{theorem}
\begin{proof}
The clearance time and associated longitudinal distance at which point the particle model reaches $y_c$ are given by
\begin{align}
t_c = \sqrt{\frac{2y_c}{a^{\text{lat}}_{\min}}}, &&
\bar{x}_c = v t_c  - \frac{a_{\min,\text{brake}} t_c^2}{2} + d_i'.
\label{eq:lower_bound_xc}
\end{align}

By the acceleration constraints imposed on the particle model,
any feasible acceleration in the bicycle model is feasible for the particle model.
In addition, the manoeuvre is optimal with respect to longitudinal distance travelled
for the particle model. Both of these points imply that the $\bar{x}_c$ in \Cref{eq:lower_bound_xc}
is a lower bound on any feasible $x_c$ for the bicycle model. Next, the inner approximation
implies that for any manoeuvre, if the chassis can clear, the square with side length
$2d_i'$ can clear as well, allowing a buffer of $d_i'$ to be added.

If we denote the initial longitudinal distance between the vehicles as $x_2$, then the distance between the swerving vehicle and the braking vehicle during the reaction delay is given by $x_2 - d_i' - d_r + v_f t - \frac{1}{2}a_{\max,\text{brake}}t^2-v_r t-\frac{1}{2}a_{\max}t^2$.
If we denote the distance between the vehicles at the end of the reaction delay as $x_{\rho}$, then after the reaction delay the distance between the vehicles is given by $x_{\rho} 
+ v_f t - \frac{1}{2}a_{\max,\text{brake}}t^2
- v_{r,\rho} t + \frac{1}{2}a_{\min,\text{brake}}t^2 $.
Since $-a_{\max,\text{brake}}-a_{\max} < 0$ and 
$-a_{\max,\text{brake}}+a_{\min,\text{brake}} < 0$, the distance between the swerving and braking vehicle is concave on both intervals. This implies that the minimum gap occurs at the boundaries of the time intervals $\{0, \rho, t_c\}$. Since the distance between the vehicles is differentiable everywhere, the time $\rho$ is a critical point only if the derivative is zero. In this case, since the distance is concave before and after time $\rho$, the derivative is positive for $t<\rho$ and negative for $t>\rho$, implying the distance at time $\rho$ is a local maximum. 
Taking everything together, assuming the vehicles are not already in collision at $t=0$, this
implies that \Cref{eq:lower_bound_swerve} is a lower bound on the longitudinal safe distance required for a swerve feasible for the bicycle model.
\end{proof}

A comparison between the lower bound and upper bound on the longitudinal distance
travelled during a swerve, as well as the equivalent braking distance, is shown
in Figure~\ref{fig:lower_upper_brake_plot}. The plot is across a range of
initial speeds.

\begin{figure}[thpb]
  \centering
  \includegraphics[scale=0.4]{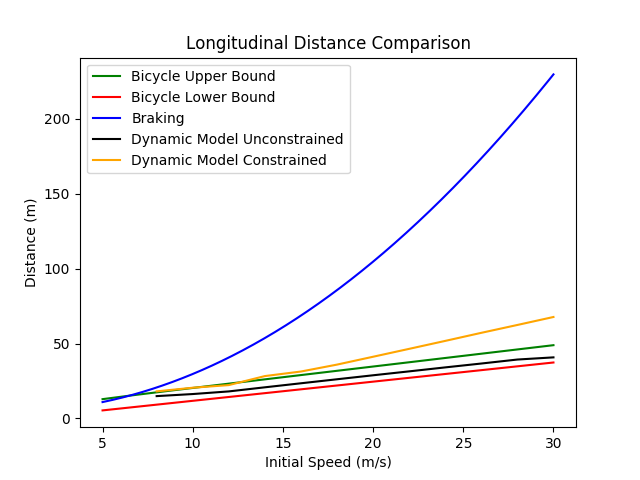}
  \caption{
    \label{fig:lower_upper_brake_plot}
    A comparison of the longitudinal distance travelled between swerve and brake manoeuvres,
    for varying initial velocities. For very low speeds the dynamic model swerves behave poorly and are omitted.
  }
\end{figure}

\subsection{Dynamic Model Validation}
\label{sec:swerve_experimental_setup}
Next, we verify that our kinematic approximation is valid by comparing the longitudinal 
swerve distance under a dynamic model to the distance computed in
the preceding sections. We analyze both the cases when the dynamic model is constrained by $a_{\min,\text{brake}}$ and $a_{\min}^{\text{lat}}$, and when it is not. We wish to show that our acceleration
constrained bicycle model swerve distances bound the swerve
manoeuvre distances of a dynamic model employing both swerving and
braking whose accelerations are unconstrained by comfort, but instead constrained by feasibility.  We would also like to see at which speeds the constrained kinematic bicycle swerve distance is close to the dynamic model swerve distance when the dynamic model is constrained by comfort.
We focus on the ability of the dynamic model to swerve, and not an associated controller, and as a result
generate the manoeuvres in open loop. However, doing a 
grid search over all possible control inputs to find the best swerves is impractical. 
Instead, we assume that the steering input is broken into 4 equal length intervals of time, 
and perform binary search over steering rate magnitudes
until the boundary conditions in \Cref{eq:bcs} are satisfied.
In addition, we also perform linear search over brake input and the total time of the
manoeuvre and select the manoeuvre that minimizes the longitudinal swerve distance $x_c$. Note that these generated swerves are not optimal for the dynamic model, but are feasible.

The parameters used in our validation are summarized in Table~\ref{table:parameters}. We chose $a_{\min,\text{brake}}$ to
represent braking at the limit of comfort, and $a_{\max, \text{brake}}$ was
chosen to represent a hard, uncomfortable brake.
The swerves generated for various initial speeds are illustrated in 
Figure~\ref{fig:dynamic_swerves}.

\begin{table}[h]
\caption{Parameters Table}
\label{table:parameters}
\begin{center}
\begin{tabular}{@{} c c c c c c @{}}
\toprule
$m$ & 1239 kg & $l_f$ & 1.19 m & $l_r$ & 1.37 m \\
\midrule
$I_{zz}$ & 1752 $\text{kg}\cdot \text{m}^2$ & $e_{SP}$ & 0.5 m & $R$ & 0.302 m \\
\midrule
$c_w$ & 0.3 & $\rho_{\text{drag}}$ & 1.25 $\frac{\text{kg}}{\text{m}^3}$ & $A$ & 1.438\\
\midrule
$B_f$ & 10.96 & $C_f$ & 1.3 & $D_f$ & 4560.4 \\
\midrule
$E_f$ & -0.5 & $B_r$ & 12.67 & $C_r$ & 1.3 \\
\midrule
$D_r$ & 3947.81 & $E_r$ & -0.5 & $a_{\max}^{\text{lat}}$ & 4.0 $\frac{\text{m}}{\text{s}^2}$ \\
\midrule
$a_{\min}^{\text{lat}}$ & 2.0 $\frac{\text{m}}{\text{s}^2}$ & $a_{\min,\text{brake}}$ & 2.0 $\frac{\text{m}}{\text{s}^2}$ & $\mu$ & 0.1 m \\
\midrule
$a_{\max,\text{accel}}$ & 2.0 $\frac{\text{m}}{\text{s}^2}$ & $a_{\max,\text{brake}}$ & 8.0 $\frac{\text{m}}{\text{s}^2}$ & $\rho$ & 0.1 s \\
\midrule
$\alpha$ & 3.7 m & $ d_r $ & 2.3 m & $d_f$ & 2.4 m  \\
\midrule
$b_r$ & 0.9 m & $b_l$ & 0.9 m & $\delta_{\max}$ & $\frac{\pi}{6}$ \\
\hline
\end{tabular}
\end{center}
\end{table}

Using these computed swerves, we then compute the lateral clearance distance $y_c$ as before
and find the longitudinal swerve distance travelled $x_c$ that occurs at time $t_c$.
Substituting this value in at \Cref{eq:d_long_swerve_for_brake,eq:d_long_brake_for_swerve}
then gives the required longitudinal safe distance for the dynamic model. For the range of
initial vehicle speeds where swerving is more efficient than braking, the longitudinal 
safe distances required for the dynamic model
are plotted and compared to those computed in Section~\ref{sec:method} in
Figure~\ref{fig:lower_upper_brake_plot}. 

\begin{figure}[thpb]
  \centering
  \includegraphics[scale=0.165]{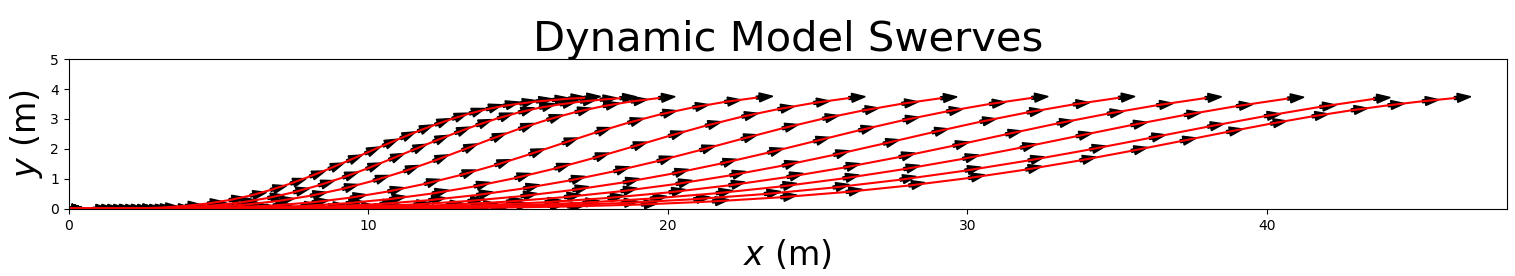}
  \caption{
    \label{fig:dynamic_swerves}
    The swerve manoeuvres generated according to the dynamic model. Each
    swerve is for a different initial speed in the interval [10, 30] $\frac{\text{m}}{\text{s}}$. The
    arrows denote the heading of the vehicle.
  }
\end{figure}

\addtolength{\textheight}{-9.5cm}   % This command serves to balance the column lengths
\subsection{Simulation Results}

In Figure~\ref{fig:lower_upper_brake_plot}, we compare the braking distance and the swerve longitudinal 
distance travelled when avoiding a stationary object. This plot illustrates the advantage
of swerves; for initial rear vehicle speeds greater than 
8 $\frac{\text{m}}{\text{s}}$, 
the swerves reach safety using less longitudinal distance than
braking does. We note that as $a_{\min,\text{brake}}$ is
increased, the crossover point of velocity where swerves become
advantageous increases as well. However, due to the quadratic
nature of the braking distance, swerves always eventually become
more advantageous at high speeds. From the figure, we can see that
when the accelerations of the dynamic model are constrained, the
swerve distance of the kinematic model is a reasonable approximation of the dynamic model, with error between 0.7-7.7\%, which is reasonable to expect for a kinematic approximation~\cite{polack_altche_dandrea-novel_fortelle_2017}.
In the case where the dynamic model is unconstrained by comfort (only by feasibility), the longitudinal swerve distance required is within 15.6-24.0\% error of the
upper bound distance of the kinematic model, and is completely bracketed by the kinematic upper and lower bounds across a range of speeds from 8-30 $\frac{\text{m}}{\text{s}}$.
This shows that our acceleration constrained kinematic approximation can accurately approximate the 
swerve distance required by the constrained dynamic
single-track model up to mid-ranged initial speeds, and can 
bound the swerve distance required by the unconstrained 
dynamic model across
the entire range of speeds.

In Figure~\ref{fig:d_safe_combined}, the universal safe following distance required as clearance when
using swerves is compared to the braking following distance, as
$a_{\min,\text{brake}}$ is varied from 2, 3, and 4
$\frac{\text{m}}{\text{s}^2}$. 
In these plots, all vehicles are moving
at an equal speed, displayed on the x-axis. These plots show that as
the speeds of the vehicles increase, the following distance decreases
when allowing swerve manoeuvres, when compared to braking alone. As
$a_{\min,\text{brake}}$ increases, the speeds where swerves become more
effective also increases. For increasing $a_{\min,\text{brake}}$ of 2, 3, and 4 $\frac{\text{m}}{\text{s}^2}$, these
speeds are 8.1, 11.4, and 14.6 $\frac{\text{m}}{\text{s}}$, respectively. The universal
following distance is also reduced by up to 42\% across all 3 values of
$a_{\min,\text{brake}}$ when using swerves as opposed to braking alone.

\begin{figure}[thpb]
\centering
  \includegraphics[scale=0.4]{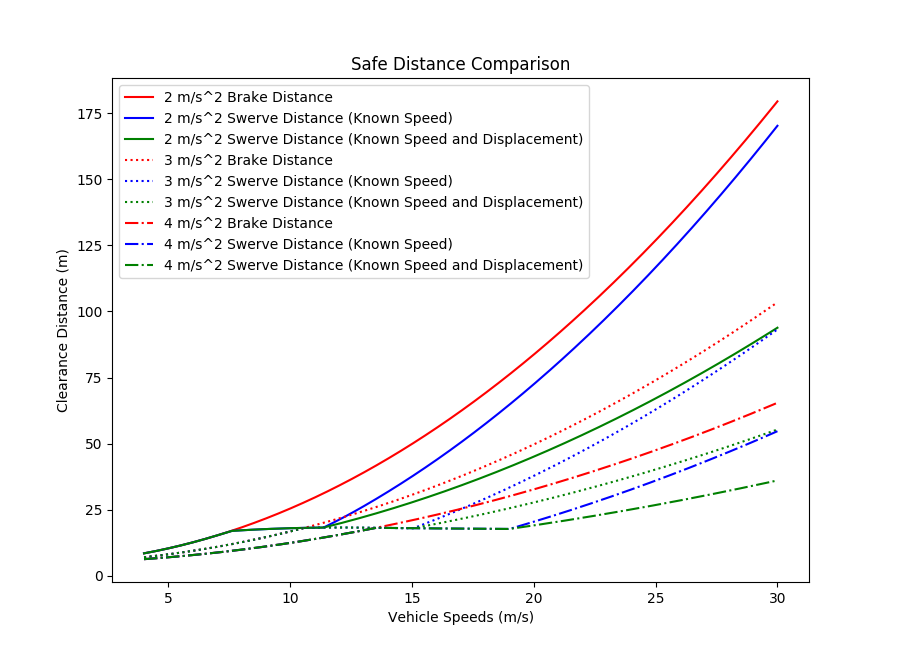}
  \caption{
    \label{fig:d_safe_combined} The universal following distance required for safety when $a_{\min,\text{brake}}$ is 2, 3, and 4 $\frac{m}{s^2}$.
  }
\end{figure}

\section{Conclusions}

In this work, we outlined a method for extending the RSS framework to include swerve 
manoeuvres in addition to the standard brake manoeuvre available in the framework.
We proved the safety of these manoeuvres under a set of reasonable assumptions about 
responsible behaviour, while incorporating the original assumptions in the RSS framework.
This extended framework results in a significant reduction in following distance at high speeds. In addition, the kinematic model was shown to conservatively bound the longitudinal distance required for swerves executed for the dynamic model.

In future work, we would like to extend the inclusion of swerve manoeuvres to
more general cases. One option would be to generalize the swerve manoeuvre to
arbitrary Frenet frames, as opposed to straight lines. One could also compute
bounds on the error from using a straight line approximation to the Frenet frame.
Further experimental work of the RSS framework and its extensions, through on-car 
testing or scenario simulation, would also be beneficial.

%%%%%%%%%%%%%%%%%%%%%%%%%%%%%%%%%%%%%%%%%%%%%%%%%%%%%%%%%%%%%%%%%%%%%%%%%%%%%%%%

%%%%%%%%%%%%%%%%%%%%%%%%%%%%%%%%%%%%%%%%%%%%%%%%%%%%%%%%%%%%%%%%%%%%%%%%%%%%%%%%

%%%%%%%%%%%%%%%%%%%%%%%%%%%%%%%%%%%%%%%%%%%%%%%%%%%%%%%%%%%%%%%%%%%%%%%%%%%%%%%%

%%%%%%%%%%%%%%%%%%%%%%%%%%%%%%%%%%%%%%%%%%%%%%%%%%%%%%%%%%%%%%%%%%%%%%%%%%%%%%%%

\bibliography{citations}{}
\bibliographystyle{IEEEtran}

\end{document}